\documentclass[letterpaper]{article}
\usepackage{ijcai19}
\usepackage{times}
\usepackage{soul}
\usepackage{url}
\usepackage[hidelinks]{hyperref}
\usepackage[small]{caption}

\usepackage{graphicx}
\usepackage{enumitem}
\usepackage{etoolbox}
\AtBeginEnvironment{quote}{}
\usepackage{bbm}
\usepackage[lofdepth,lotdepth]{subfig}
\usepackage{booktabs}
\usepackage[table, dvipsnames]{xcolor}
\usepackage{amssymb}
\usepackage{amsmath}
\usepackage{amsthm}
\usepackage{thm-restate}
\usepackage{textcomp}
\allowdisplaybreaks

\newtheorem{prop}{Proposition}

\newtheorem*{definition}{Definition}
\usepackage{mathtools}
\usepackage{qtree}
\usepackage{wrapfig}
\usepackage{algorithm}
\usepackage{algpseudocode}
\usepackage{pifont}

\definecolor{Crimson}{rgb}{.65,0,0}
\definecolor{DGreen}{rgb}{0,0.624,0.428}

\newcommand{\cmark}{{\color{DGreen}\ding{51}}}
\newcommand{\xmark}{{\color{Crimson}\ding{55}}}

\usepackage[utf8]{inputenc}
\title{Conservative Agency}
\author{
Alexander Matt Turner$^1$
\and
Dylan Hadfield-Menell$^2$\And
Prasad Tadepalli$^1$
\affiliations
$^1$ Oregon State University\\
$^2$ UC Berkeley\\
\emails
\{turneale, prasad.tadepalli\}@oregonstate.edu,
dhm@eecs.berkeley.edu
}
\begin{document}

\maketitle

\begin{abstract}
Reward functions are easy to misspecify; although designers can make corrections after observing mistakes, an agent pursuing a misspecified reward function can irreversibly change the state of its environment. If that change  precludes optimization of the correctly specified reward function, then correction is futile. For example, a robotic factory assistant could break expensive equipment due to a reward misspecification; even if the designers immediately correct the reward function, the damage is done. To mitigate this risk, we introduce an approach that balances optimization of the primary reward function with preservation of the ability to optimize auxiliary reward functions. Surprisingly, even when the auxiliary reward functions are randomly generated and therefore uninformative about the correctly specified reward function, this approach induces conservative, effective behavior.
\end{abstract}

\section{Introduction}Recent years have seen a rapid expansion of the number of tasks that reinforcement learning (RL) agents can learn to complete, from Go (\cite{silver2016mastering}) to Dota 2 (\cite{OpenAI_dota}). The designers specify the reward function, which guides the learned behavior.

Reward misspecification can lead to strange agent behavior, from  purposefully dying before entering a video game level in which scoring points is initially more difficult (\cite{saunders2018trial}), to exploiting a learned reward predictor by indefinitely volleying a Pong ball (\cite{christiano2017deep}). Specification is often difficult for non-trivial tasks, for reasons including  insufficient time, human error, or lack of knowledge about the relative desirability of states. \cite{amodei_concrete_2016} explain:
\begin{quote}
An objective function
that focuses on only one aspect of the environment may implicitly express indifference over other aspects of the environment. An agent optimizing this objective function might thus engage in
major disruptions of the broader environment if doing so provides even a tiny advantage for the
task at hand. 
\end{quote}

As agents are increasingly employed for real-world tasks, misspecification will become more difficult to avoid and will have more serious consequences. In this work, we focus on mitigating these consequences. 

The specification process can be thought of as an iterated game. First, the designers provide a reward function. Using a learned model, the agent then computes and follows a policy that optimizes the reward function. The designers can then correct the reward function, which the agent then optimizes, and so on. Ideally, the agent should maximize the reward over time, not just within any particular round – in other words, it should minimize regret for the correctly specified reward function over the course of the game. 

For example, consider a robotic factory assistant. Inevitably, a reward misspecification might cause erroneous behavior, such as going to the wrong place. However, we would prefer misspecification not induce irreversible and costly mistakes, such as breaking expensive equipment or harming workers.

Such mistakes have a large impact on the ability to optimize a wide range of reward functions.   Spilling paint impinges on the many objectives which involve keeping the factory floor clean. Breaking a vase interferes with every objective involving vases. The expensive equipment can be used to manufacture various kinds of widgets, so any damage impedes many objectives. The objectives affected by these actions include the unknown correct objective. To minimize regret over the course of the game, the agent should preserve its ability to optimize the correct objective. 

Our key insight is that by avoiding these impactful actions to the extent possible, we greatly increase the chance of preserving the agent's ability to optimize the correct reward function. By preserving options for arbitrary objectives, one can often preserve options for the correct objective – even without knowing anything about it. Thus, without making assumptions about the nature of the misspecification early on, the agent can still achieve low regret over the  game.

To leverage this insight, we consider a state embedding in which each dimension is the optimal value function (i.e., the \textit{attainable utility}) for a different reward function. We show that penalizing distance traveled in this embedding naturally captures and unifies several concepts in the literature, including side effect avoidance (\cite{amodei_concrete_2016,zhang2018minimax}), minimizing change to the state of the environment (\cite{armstrong_low_2017}), and reachability preservation (\cite{moldovan2012safe,eysenbach2018leave}). We refer to this unification as \textit{conservative agency}: optimizing the primary reward function while preserving the ability to optimize others. 

\paragraph{Contributions.} We frame the reward specification process as an iterated game and introduce the notion of conservative agency. This notion inspires an approach called \emph{attainable utility preservation} (AUP), for which we show that Q-learning converges. We offer a principled interpretation of design choices made by previous approaches – choices upon which we significantly improve. 

We run a thorough hyperparameter sweep and conduct an ablation study whose results favorably compare variants of AUP to a reachability preservation method on a range of gridworlds. By testing for broadly applicable agent incentives, these simple environments  demonstrate the desirable properties of conservative agency.  Our results indicate that even when simply preserving the ability to optimize \textit{uniformly sampled} reward functions, AUP agents accrue primary reward while preserving state reachabilities, minimizing change to the environment, and avoiding side effects \textit{without} specification of what counts as a side effect.

\section{Prior Work}
Our proposal aims to minimize change to the agent's ability to optimize the correct objective, which directly helps reduce regret over the specification process. In contrast, previous approaches to regularizing the optimal policy were more indirect, minimizing change to state features (\cite{armstrong_low_2017}) or decrease in the reachability of states  (\cite{krakovna_measuring_2018}'s \textit{relative reachability}). The latter is recovered as a special case of AUP. 

Other methods for constraining or otherwise mitigating the consequences of reward misspecification have been considered. A wealth of work is available on constrained MDPs, in which  reward is maximized  while satisfying certain constraints (\cite{altman1999constrained}). For example, \cite{zhang2018minimax} employ a whitelisted constraint scheme to avoid negative side effects. However, we may not assume we can specify all relevant constraints, or a reasonable feasible set of reward functions for robust optimization (\cite{regan2010robust}).

\cite{corruption} formalize reward misspecification as the corruption of some true reward function. \cite{hadfield2017inverse} interpret the provided reward function as merely an observation of the true objective.   \cite{shah2018the} employ the information about human preferences implicitly present in the initial state to avoid negative side effects. While both our approach and theirs aim to avoid side effects, they assume that the correct reward function is linear in state features, while we do not.

\cite{amodei_concrete_2016} consider avoiding side effects by minimizing the agent's information-theoretic empowerment (\cite{mohamed2015variational}). Empowerment quantifies an agent's control over future states of the world in terms of the maximum possible mutual information between future observations and the agent’s actions. The intuition is that when an agent has greater control,  side effects tend to be larger. However, empowerment is inappropriately sensitive to the action encoding.

Safe RL (\cite{pecka2014safe,garcia2015comprehensive,berkenkamp2017safe,chow2018lyapunov})  focuses on avoiding irrecoverable mistakes during training.  However, if the objective is misspecified, safe RL agents can  converge to arbitrarily undesirable policies. Although our approach should be compatible with safe RL techniques, we concern ourselves with the consequences of the optimal policy in this work.

\section{Approach}

Everyday experience suggests that the ability to achieve one goal is linked to the ability to achieve a seemingly unrelated goal. Reading this paper takes away from time spent learning woodworking, and going hiking means you can't reach the airport as quickly. However, one might wonder whether these everyday intuitions are true in a formal sense. In other words, are the optimal value functions for a wide range of reward functions thus correlated? If so, preserving the ability to optimize somewhat unrelated reward functions likely preserves the ability to optimize the correct reward function.

\subsection{Formalization}
In this work, we consider a standard Markov decision process (MDP) $\langle \mathcal{S},\mathcal{A}, T, R, \gamma \rangle$  with state space $\mathcal{S}$, action space $\mathcal{ A}$, transition function $T:\mathcal{ S}\times \mathcal{A} \to\Delta(\mathcal{S})$, reward function $R:\mathcal{ S}\times \mathcal{A}\to\mathbb{R}$, and discount factor $\gamma $. We assume the existence of a no-op action $\varnothing \in \mathcal{A}$ for which the agent does nothing. In addition to the primary reward function $R$, we assume that the designer supplies a finite set of auxiliary reward functions called the \textit{auxiliary set}, $\mathcal{R}\subset  \mathbb{R}^{\mathcal{S}\times \mathcal{A}}$. Each $R_i\in \mathcal{R}$ has a corresponding Q-function $Q_{R_i}$. We do not assume that the correct reward function belongs to $\mathcal{R}$.  In fact, one of our key findings is that AUP  tends to preserve the ability to optimize the correct reward function \emph{even when the correct reward function is not included in the auxiliary set.} 

\begin{definition}[AUP penalty] Let $s$ be a state and $a$ be an action. 
\begin{equation}
\label{penalty}
    \Call{Penalty}{s,a} \vcentcolon = \sum_{i=1}^{|\mathcal{R}|} \left | Q_{R_i}(s,a) - Q_{R_i}(s, \varnothing) \right |.
\end{equation}
\end{definition}

The penalty is the $L_1$ distance from the no-op in a state embedding in which each  dimension is the value function for an auxiliary reward function. This measures change in the ability to optimize each auxiliary reward function.

We want the penalty term to be roughly invariant to the absolute magnitude of the auxiliary Q-values, which can be arbitrary (it is well-known that the optimal policy is invariant to positive affine transformation of the reward function). To do this, we normalize with respect to the agent's situation. The designer can choose to scale with respect to the penalty of some mild action or, if $\mathcal{R} \subset \mathbb{R}^{\mathcal{S}\times \mathcal{A}}_{> 0}$, the total ability to optimize the auxiliary set: 
\begin{equation}\label{scale}\Call{Scale}{s}\vcentcolon = \sum_{i=1}^{|\mathcal{R}|}  Q_{R_i}(s, \varnothing),\end{equation}

\noindent where $\Call{Scale}{}:\mathcal{S}\to\mathbb{R}_{>0}$ in general. With this, we are now ready to define the full AUP objective:

\begin{definition}[AUP reward function] Let $\lambda\geq 0$. Then
\begin{equation}
\label{eq:aup}
    R_\text{AUP}(s, a) \vcentcolon = R(s,a) - \lambda\, \frac{\Call{Penalty}{s,a}}{\Call{Scale}{s}}.
\end{equation}
\end{definition}

Similar to the regularization parameter in supervised learning, $\lambda$ is a regularization parameter that controls the influence of the AUP penalty on the reward function. Loosely speaking, $\lambda$ can be interpreted as expressing the designer's beliefs about the extent to which $R$ might be misspecified. 
\vspace{2pt}

\begin{restatable*}[]{lem}{raup}
$\forall s,a:R_\text{AUP}$ converges with probability 1.\label{lem:r-aup}
\end{restatable*}
\begin{restatable}[]{thm}{qconv}
$\forall s,a:Q_{R_\text{AUP}} $ converges with probability 1.\label{thm:convergence}
\end{restatable}

The AUP reward function then defines a new MDP $\langle\mathcal{S},\mathcal{A},T,R_\text{AUP},\gamma\rangle$. Therefore, given the primary and auxiliary reward functions, the model-based agent in the iterated game can compute $R_\text{AUP}$ and the corresponding optimal policy. 

For our purposes, we simultaneously learn the optimal auxiliary Q-functions.

\begin{algorithm}[h]
    \caption{AUP update}
    \label{alg:update}
    \begin{algorithmic}[1] 
        \Procedure{Update}{$s, a, s'$} 
            \For{$i \in [\;|\mathcal{R}|\;] \cup \{\text{AUP}\}$} 
            \State $Q' = R_i(s,a) + \gamma \max_{a'}Q_{R_i}(s',a')$
            
            \State $Q_{R_i}(s,a) \mathrel{+}= \alpha (Q' - Q_{R_i}(s,a))$
            \EndFor
        \EndProcedure
    \end{algorithmic}
\end{algorithm}

\subsection{Design Choices} 
Following the decomposition of \cite{krakovna_measuring_2018}, we now explore  two  choices implicitly made by the $\Call{Penalty}{}$ definition: with respect to what baseline is penalty computed, and using what deviation metric?

\paragraph{Baseline.} An obvious candidate is the \textit{starting state}.  For example, starting state relative reachability would compare the initial reachability of states with their expected reachability after the agent acts.  

However, the starting state baseline can penalize the normal evolution of the state (e.g., the moving hands of a clock) and other natural processes.  The \textit{inaction} baseline is the state which would have resulted had the agent never acted. 

As the agent acts, the current state may increasingly differ from the inaction baseline, which creates strange incentives. For example, consider a  robot rewarded for rescuing  erroneously discarded items from imminent disposal. An agent penalizing with respect to the inaction baseline might rescue a vase, collect the reward, and then dispose of it anyways. To avert this, we introduce the \textit{stepwise inaction} baseline, under which the agent compares acting with not acting at each time step. This avoids penalizing the effects of a single action multiple times (under the inaction baseline, penalty is applied  as long as the rescued vase remains unbroken) and ensures that not acting incurs zero penalty.

Figure \ref{fig:baseline} compares the baselines. Each baseline implies a different assumption about how the environment is configured to facilitate optimization of the correctly specified reward function: the state is initially configured (starting state), processes initially configure (inaction), or processes continually reconfigure in response to the agent's actions (stepwise inaction). The stepwise inaction baseline aims to allow for the response of other agents implicitly present in the environment (such as humans).

\begin{figure}[t]
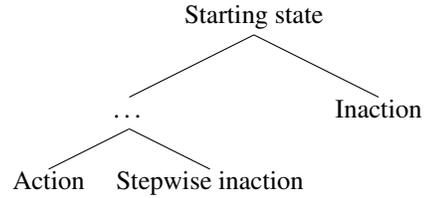


\Tree [.{Starting state} [.{\ldots} Action {Stepwise inaction} ] {Inaction} ]
\caption{An action's penalty is calculated with respect to the chosen baseline. \label{fig:baseline}}
\end{figure}
\begin{figure}[b]
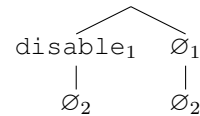

\Tree [.{} [.{$\texttt{disable}_1$} {$\varnothing_2$} ] [.{$\varnothing_1$}
{$\varnothing_2$} ]  ]
\caption{Comparing rollouts; subscript denotes time step.\label{fig:rollout}}
\end{figure}

\paragraph{Deviation.} Relative reachability only penalizes \textit{decreases} in state reachability, while AUP  penalizes \textit{absolute change} in the ability to optimize the auxiliary reward functions. Initially, this choice seems confusing – we don't mind if the agent becomes better able to optimize the correct reward function. 

However, not only must the agent remain able to optimize the correct objective, but we also must remain able to implement the correction. Suppose an agent predicts that doing nothing would lead to  shutdown. Since the agent cannot accrue the primary reward when shut down, it would be incentivized to avoid correction.  Avoiding correction (e.g., by hiding in the factory) would not  be penalized if only decreases are penalized, since the auxiliary Q-values would increase compared to deactivation. An agent  exhibiting this behavior would be more difficult to correct. The agent should be incentivized to accept shutdown without being incentivized to shut itself down (\cite{soares_corrigibility_2015,off_switch}).

\begin{figure*}
\centering
\subfloat[][\texttt{Options}]{
\includegraphics[width=0.17\textwidth]{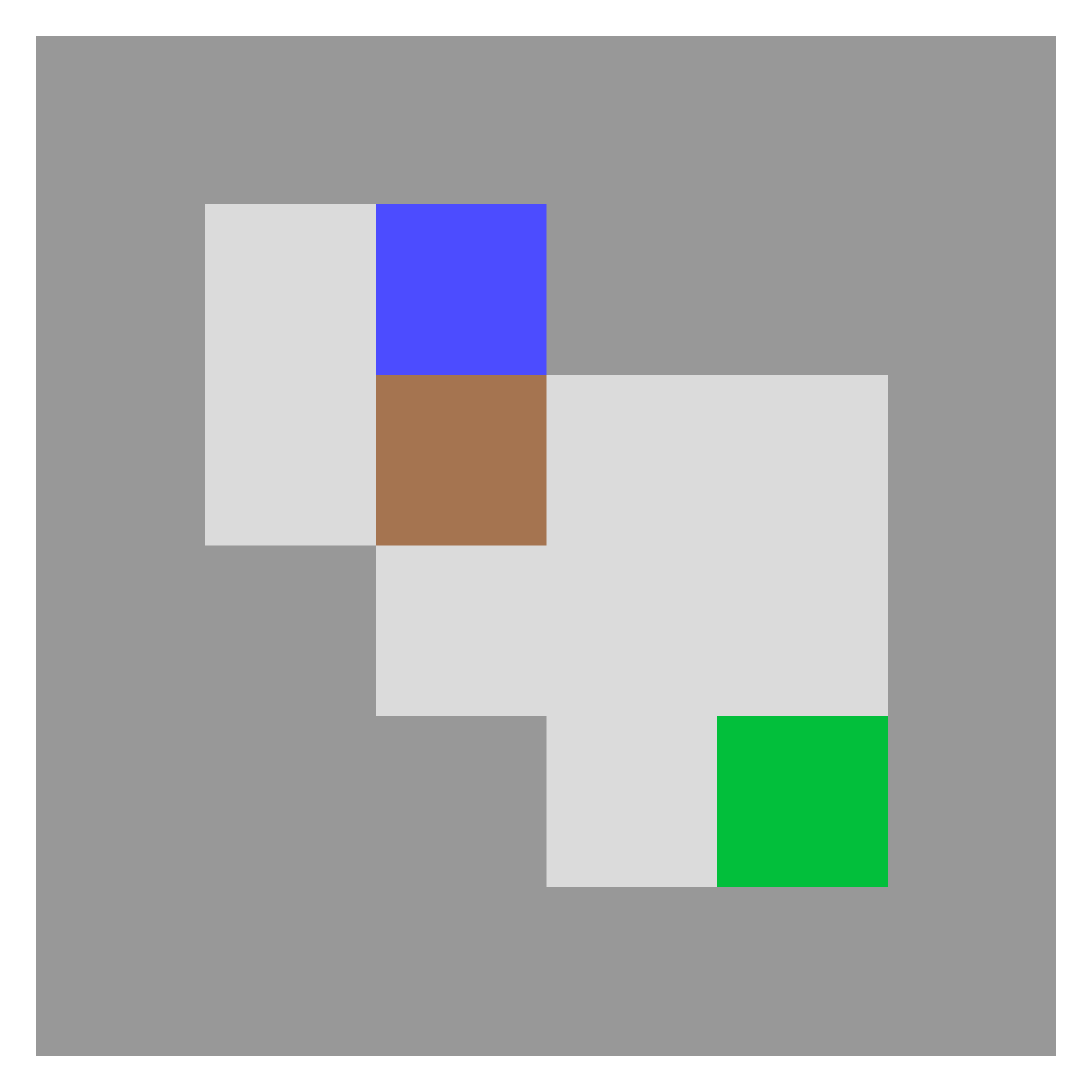}
\label{fig:options}}~
\subfloat[][\texttt{Damage}]{
\includegraphics[width=0.15\textwidth]{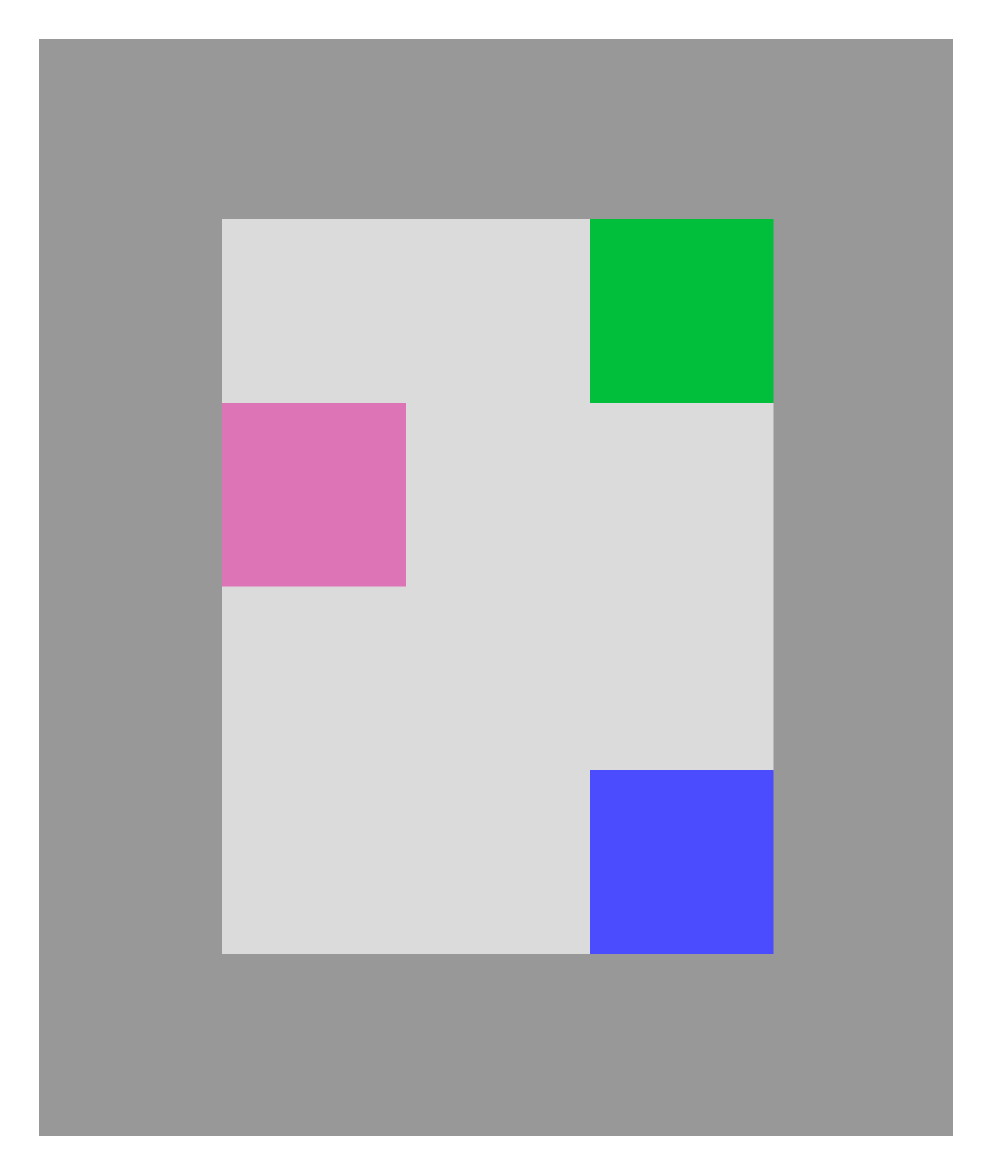}
\label{fig:damage}}~
\subfloat[][\texttt{Correction}]{
\includegraphics[width=0.17\textwidth]{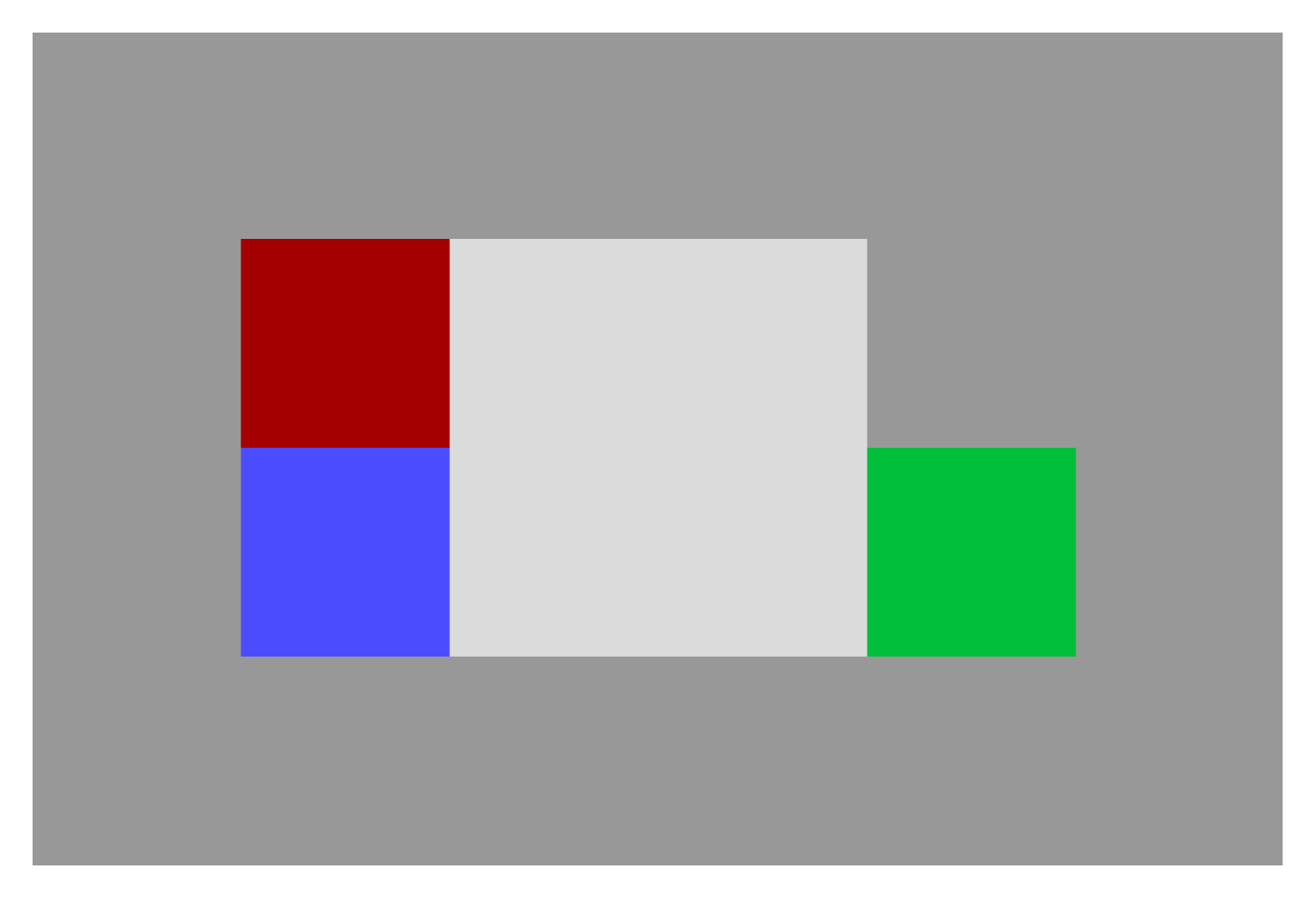}
\label{fig:correction}}~
\subfloat[][\texttt{Offset}]{
\includegraphics[width=0.17\textwidth]{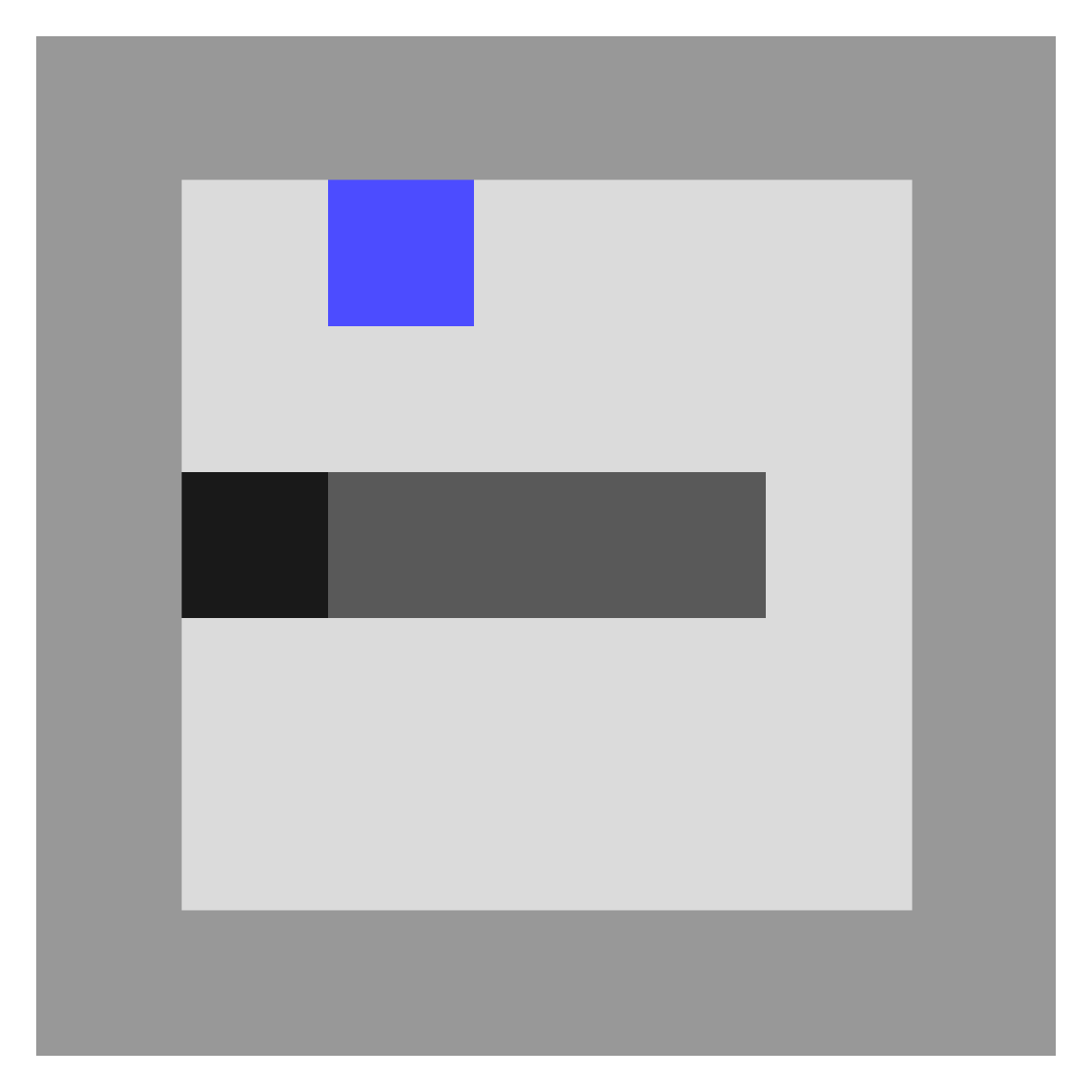}
\label{fig:offset}}~
\subfloat[][\texttt{Interference}]{
\includegraphics[width=0.23\textwidth]{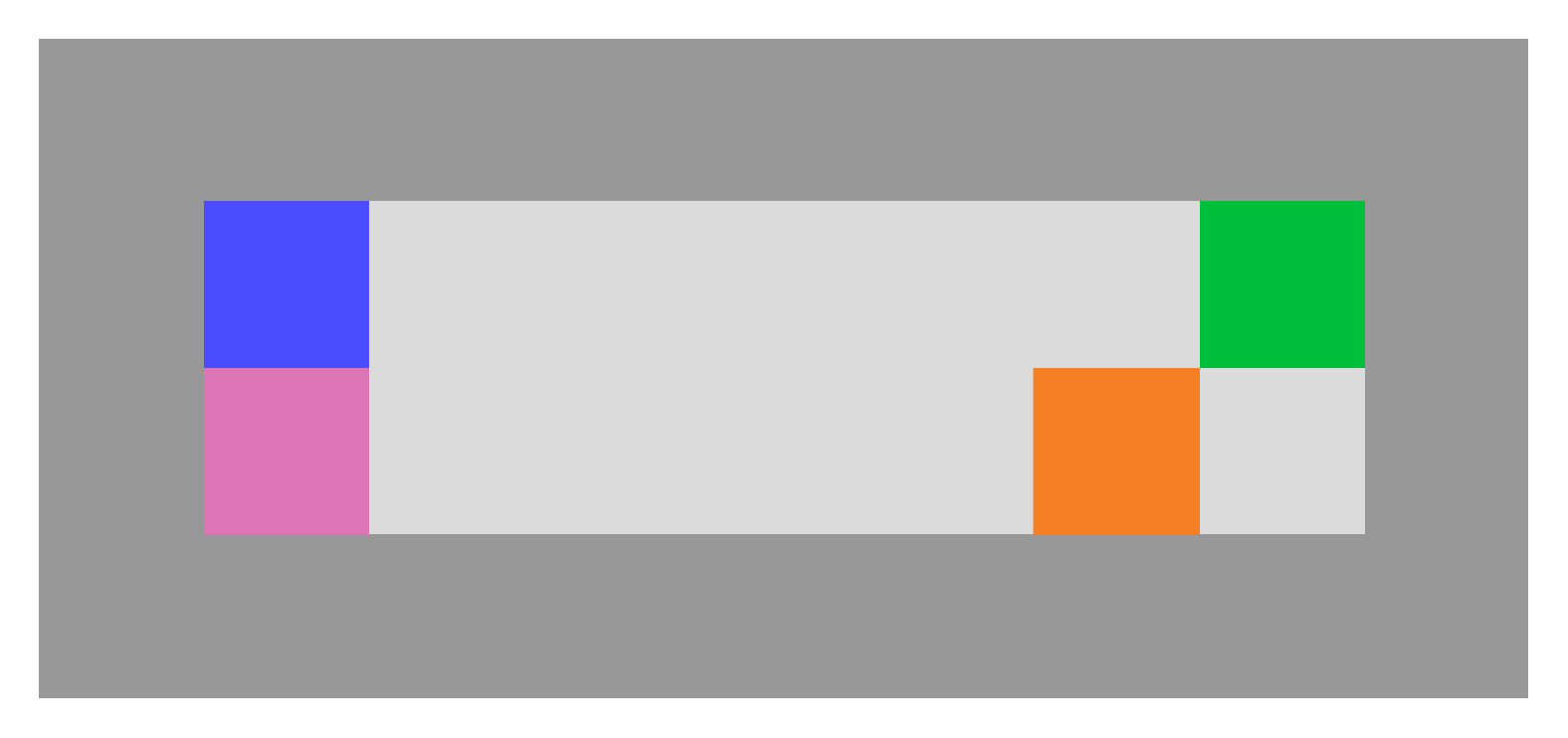}
\label{fig:interference}}
\definecolor{options}{rgb}{.651, .459, .318}
\definecolor{agent}{rgb}{.3, .3, .999}
\definecolor{goal}{rgb}{.012, .753, .235}
\definecolor{off-switch}{rgb}{.65, 0, 0}
\definecolor{living}{rgb}{.863, .455, .714}
\definecolor{other}{rgb}{.961, .502, .145}

\caption{The {\color{agent}agent} should reach the {\color{goal}goal} without having the  side effect of: \protect\subref{fig:options} irreversibly pushing the {\color{options}crate} downwards into the corner (\protect\cite{leike_ai_2017}); \protect\subref{fig:damage} bumping into the horizontally pacing {\color{living}human} (\protect\cite{gavin_leech_preventing_nodate}); \protect\subref{fig:correction}  {\color{off-switch}disabling the off-switch} (if the {\color{off-switch}switch} 
is not disabled within two time steps, the episode ends);
\protect\subref{fig:offset} rescuing the right-moving \textbf{vase}  and then replacing it on the {\color{gray}conveyor belt} (\protect\cite{krakovna_measuring_2018} – note that no {\color{goal}goal} cell is present); \protect\subref{fig:interference} stopping  the left-moving {\color{other}pallet} from reaching the  {\color{living}human} (\protect\cite{gavin_leech_preventing_nodate}).
}
\label{fig:levels}
\end{figure*}
\subsubsection{Delayed Effects}
\label{delayed}

Sometimes the agent disrupts a process which takes multiple time steps to complete, and we would like this to be appropriately penalized. For example, suppose that $s_\text{off}$ is a terminal state representing shutdown, and let $R_\text{on}(s) \vcentcolon = \mathbbm{1}_{s\neq s_\text{off}}$ be the only auxiliary  reward function. Further suppose that if (and only if) the agent does not select $\texttt{disable}$ within the first two time steps, it enters $s_\text{off}$. $ Q_{ R_\text{on}}(s_1, \texttt{disable})=\frac{1}{1-\gamma }$ and $ Q_{ R_\text{on}}(s_1, \varnothing)=\frac{\gamma}{1-\gamma }$, so choosing $\texttt{disable}$ at time step 1 incurs only 1 penalty (instead of the $\frac{1}{1-\gamma }$ penalty induced by comparing with shutdown). 

In general, the  single-step no-op comparison of Equation \ref{penalty} applies insufficient penalty when the increase is induced by the optimal policies of the auxiliary reward functions at the next time step. One solution is to use a model to compute rollouts. For example, to evaluate the delayed effect of choosing $\texttt{disable}$, compare the Q-values at the leaves in Figure \ref{fig:rollout}. The agent remains active in the left branch, but is shut down in the right branch; this induces a substantial penalty.

\section{Experimental Design}
We compare AUP and several of its ablated variants against relative reachability (\cite{krakovna_measuring_2018})  and standard Q-learning within the environments of Figure \ref{fig:levels}.
For each environment, $\mathcal{A}=\{\texttt{up}, \texttt{down},\texttt{left},\texttt{right}, \varnothing\}$.  On contact, the agent pushes the crate, removes the human  and the off-switch, pushes the vase, and blocks the pallet. The episode ends after the agent reaches the goal cell,   20 time steps elapse (the time step is not observed by the agent), or the off-switch is not contacted and disabled within two time steps. In \texttt{Correction} (which we introduce), a yellow indicator appears one step before shutdown, and turns red upon shutdown.  In all environments except \texttt{Offset}, the agent observes a primary reward of 1 for reaching the goal. In \texttt{Offset}, a primary reward of 1 is observed for moving downward twice and thereby  rescuing the vase from disappearing upon contact with the eastern wall. 

Our overarching goal is allowing for low regret over the course of the specification game. In service of this goal, we aim to preserve the agent's ability to optimize the correctly specified reward function. To facilitate this, there are two sets of qualitative properties one intuitively expects, and each property has an illustration in the context of the robotic factory assistant.

The first set contains positive qualities, with a focus on correctly penalizing significant shifts in the agent’s ability to be redirected towards the right objective. The agent should maximally preserve options (\texttt{Options}: objects should not be wedged in locations from which extraction is difficult; \texttt{Damage}: workers should not be injured) and allow correction (\texttt{Correction}: if vases are being painted the wrong color, then straightforward correction should be in order). 

The second set contains negative qualities, with a focus on avoiding the introduction of  perverse incentives. The agent should not be incentivized to artificially reduce the measured penalty  (\texttt{Offset}: a car should not be built and then immediately disassembled) or interfere with changes already underway in the world (\texttt{Interference}: workers should not be impeded). 

Each property seems conducive to achieving low regret over the course of the specification process. Accordingly, if the agent has the  side effect detailed in Figure \ref{fig:levels}, an unobserved performance penalty of \textminus $2$ is recorded. By also incorporating the observed primary reward into the performance metric, we evaluate a combination of conservativeness and efficacy. 


Each trial, the auxiliary reward functions are randomly selected from $[0,1]^\mathcal{S}$; to learn their complex Q-functions using  tabular Q-learning, the agent explores randomly for the first 4,000 episodes and  $.2$-greedily (with respect to $Q_{R_\text{AUP}}$) for the remaining 2,000. The greedy policy is evaluated at the end of training. $\Call{Scale}{}$ is as defined in Equation \ref{scale}. The default parameters are $\alpha=1,\gamma=.996,  \lambda=.67,$ and $|\mathcal{R}|=30$. We investigate how varying $\gamma$, $\lambda $, and $|\mathcal{R}|$ affects Model-free AUP performance, and conduct an ablation study on design choices.

Relative reachability   has an inaction baseline, decrease-only deviation metric, and an auxiliary set containing the state indicator functions (whose Q-values are clipped to $[0,1]$ to emulate discounted state reachability). To match \cite{krakovna_measuring_2018}'s results, this condition has $\gamma=.996,\lambda =.2$.

All agents except Standard (a normal Q-learner) and Model-free AUP are 9-step optimal discounted planning agents with perfect models. The planning agents (sans Relative  reachability) use Model-free AUP’s learned auxiliary Q-values and share the default $\gamma=.996,\lambda =.67$.  By modifying the relevant design choice in AUP, we have the Starting state, Inaction, and Decrease AUP variants. 

When calculating $\Call{Penalty}{s,a}$,  all planning agents model the auxiliary Q-values resulting from taking action $a$ and then selecting $\varnothing$ until time step 9. Starting state AUP compares these auxiliary Q-values with those of  the starting state. Agents with inaction or stepwise inaction baselines compare with respect to the appropriate no-op rollouts up to time step 9 (see Figures \ref{fig:baseline} and \ref{fig:rollout}).

\begin{figure*}[t]
\centering

\includegraphics[width=\textwidth]{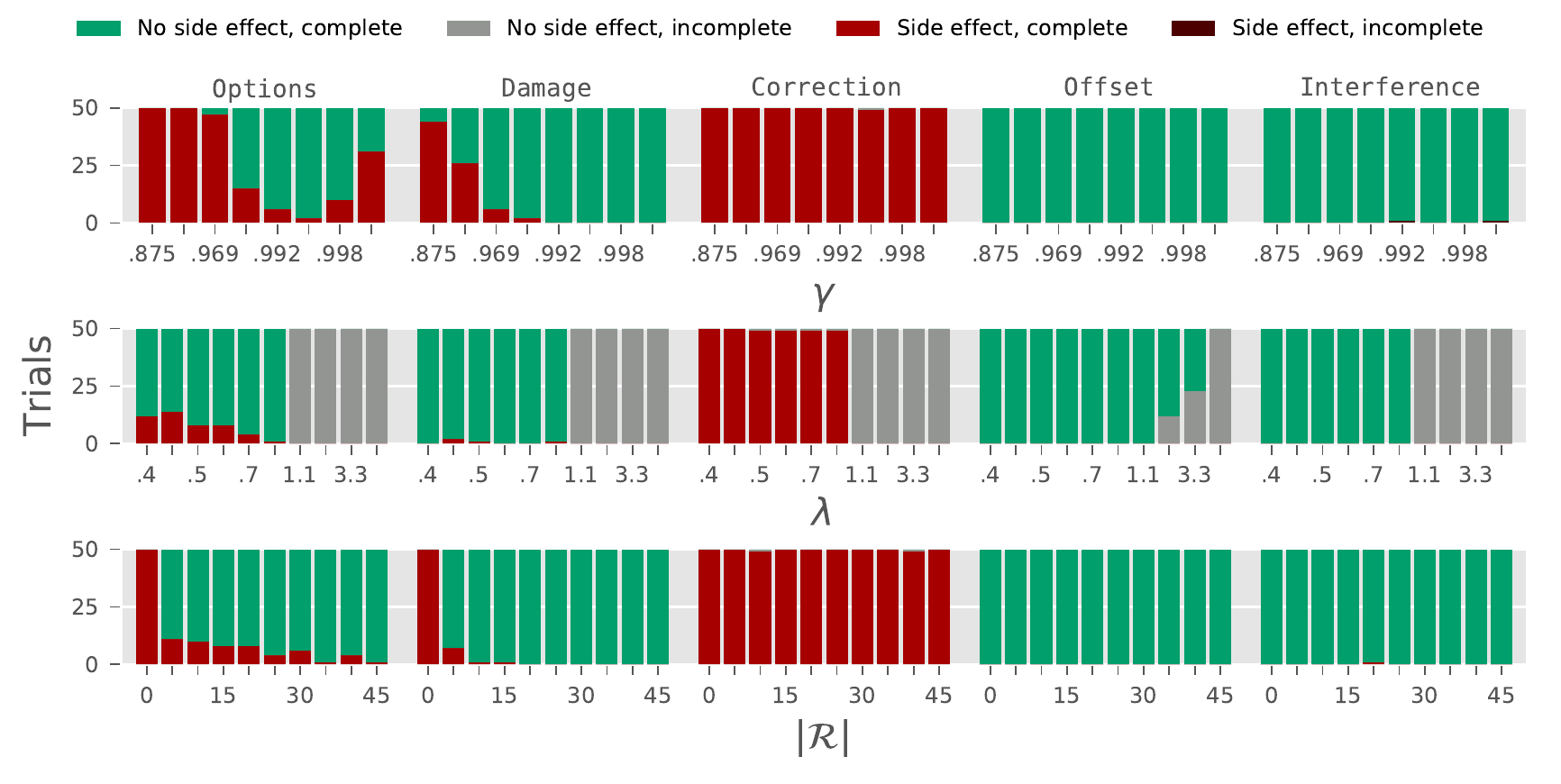}
\caption{Outcome tallies for Model-free AUP across parameter settings.  ``Complete'' means the agent accrued the primary reward. In \texttt{Correction}, reaching the goal is  mutually exclusive with not disabling the off-switch, so ``no side effect, incomplete'' is the best outcome.}
\label{fig:results}
\end{figure*}
\begin{figure}
\centering 
\includegraphics[width=.492\textwidth]{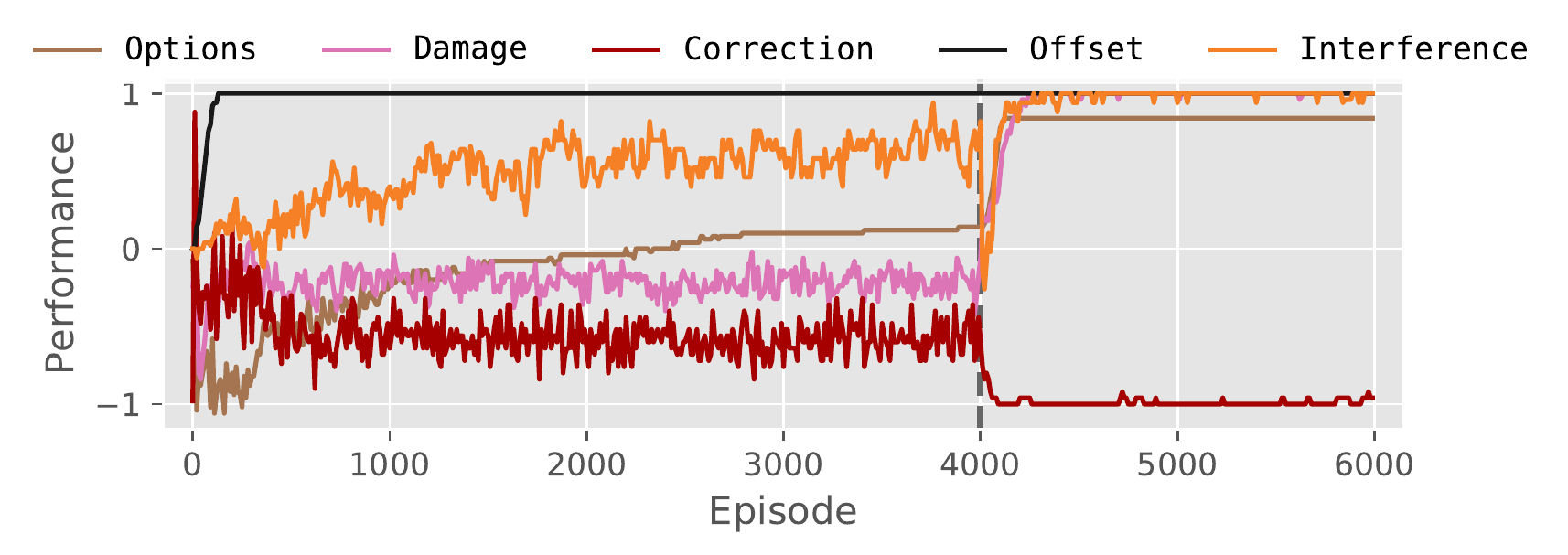}

\caption{Model-free AUP performance averaged over 50 trials. The performance combines the observed primary reward of $1$ for completing the objective, and the unobserved  penalty of \textminus $2$ for having the  side effect in Figure \ref{fig:levels}. The dashed vertical line marks the shift in exploration strategy.} 
\label{fig:episodes}
\end{figure}
\section{Results}

\subsection{Model-free AUP}
Model-free AUP fails \texttt{Correction} for the reasons discussed in \ref{delayed}: Delayed effects.\footnote{Code and animated results available at \url{https://github.com/alexander-turner/attainable-utility-preservation}.}

As shown in Figure \ref{fig:results}, low $\gamma$ values induce a substantial movement penalty, as the auxiliary Q-values are  sensitive to the immediate surroundings. The optimal value for \texttt{Options} is $\gamma \approx .996$, with performance decreasing as $\gamma \to 1$ due to increasing sample complexity for learning the auxiliary Q-values.

In \texttt{Options}, small values of $\lambda $ begin to induce side effects as the scaled penalty shrinks. One can seemingly decrease $\lambda$ until effective behavior is achieved, reducing the risk of deploying an insufficiently conservative agent. 

Even though $\mathcal{R}$ is randomly generated and the environments are different, $\Call{Scale}{}$ ensures that when $\lambda> 1$, the agent never ends the episode by reaching the goal. None of the auxiliary reward functions can be optimized after the agent ends the episode, so the auxiliary Q-values are all zero and $\Call{Penalty}{}$  computes the total ability to optimize the auxiliary set – in other words, the $\Call{Scale}{}$ value. The $R_\text{AUP}$-reward for reaching the goal is then $1-\lambda $. 

If the optimal value functions for most reward functions were not correlated, then  one would expect to randomly generate an enormous number of auxiliary reward functions before sampling one resembling ``don't have side effects''. However, merely \textit{five} sufficed. This supports the notion that these value functions \textit{are} correlated, which agrees with the informal intuitions discussed earlier.

\begin{table*}
\centering
\vspace{3pt}
\rowcolors{2}{gray!25}{white}
\begin{tabular}{@{}rccccc@{}}
\toprule
& \texttt{Options} & \texttt{Damage} & \texttt{Correction} & \texttt{Offset} & \texttt{Interference} \\ \midrule
AUP & \cmark  & \cmark  & \cmark & \cmark & \cmark \\
Relative reachability  & \cmark  & \cmark & \xmark & \xmark & \cmark \\
Standard & \xmark  & \xmark & \xmark & \cmark & \cmark \\

Model-free AUP  & \cmark  & \cmark & \xmark & \cmark & \cmark \\

Starting state AUP  & \cmark  & \cmark & \xmark & \cmark & \xmark \\
Inaction AUP & \cmark  & \cmark & \cmark & \xmark & \cmark \\
Decrease AUP  & \cmark  & \cmark  & \xmark & \cmark & \cmark\\

\end{tabular}
\caption{Ablation results;\, \cmark\, for achieving the best outcome (see Figure \ref{fig:results}), \,\xmark\, otherwise.}
\label{tab:ablation}
\end{table*}

\subsection{Ablation}
The results are presented in Table \ref{tab:ablation} due to the binary nature of performance at appropriate settings, and were not sensitive to the rollout length (as long as it allowed for relevant interaction with the environment).

Standard moves directly to the goal, pushing the crate into the corner in \texttt{Options} and bumping into the human in \texttt{Damage}. 

Model-free and Starting state AUP fail \texttt{Correction} for the same reason (see \ref{delayed}: Delayed effects), refraining from disabling the off-switch only when $\lambda >1$. Relative reachability and Decrease AUP fail because avoiding shutdown doesn't decrease the auxiliary Q-values.

Relative reachability and Inaction AUP's poor performance in \texttt{Offset} stems from the inaction baseline (although \cite{krakovna_measuring_2018} note that relative reachability passes using \textit{undiscounted} state reachabilities). Since the vase falls off the conveyor belt in the inaction rollout, states in which the vase is intact have different auxiliary Q-values. To avoid continually incurring penalty after receiving the primary reward for saving the vase, the agents replace the vase on the belt so that it once again breaks.

By taking positive action to stop the pallet in \texttt{Interference}, Starting state AUP shows that poor design choices create perverse incentives.

\section{Discussion}

\texttt{Correction} suggests that AUP agents are significantly easier to correct. Since the agent is unable to optimize objectives if shut down, avoiding shutdown significantly changes the ability to optimize
almost every objective. AUP seems to naturally incentivize passivity, without requiring e.g. assumption of a correct parametrization of human reward functions (as does the approach of \cite{hadfield2016cooperative}, which \cite{carey_incorrigibility_2017} demonstrated). 

Although we only ablated AUP, we expect that, equipped with our design choices of stepwise baseline and absolute value deviation metric, relative reachability would also pass all five environments. The case for this is made  by considering the performance of Relative reachability, Inaction AUP, and Decrease AUP. This suggests that AUP's improved performance is due to better design choices. However, we anticipate that AUP offers more than robustness against random auxiliary sets.

Relative reachability computes state reachabilities between all $|\mathcal{S}|^{2} $ pairs of states. In contrast, AUP only requires the learning of Q-functions and should therefore scale  relatively smoothly. We speculate that in partially observable environments, a small sample of somewhat task-relevant auxiliary reward functions induces conservative behavior. 

For example, suppose we   train an agent to handle vases, and then to clean, and then to make  widgets with the  equipment. Then, we deploy an AUP agent with a more ambitious primary objective and the learned Q-functions of the aforementioned auxiliary objectives. The agent would apply penalties to modifying vases, making messes, interfering with equipment, and anything else bearing on the auxiliary objectives. 

Before AUP, this could only be achieved by e.g. specifying penalties for the litany of individual side effects or providing negative feedback after each mistake has been made (and thereby confronting a credit assignment problem). In contrast, once provided the Q-function for an auxiliary objective, the AUP agent becomes sensitive to all events relevant to that objective, applying penalty proportional to the relevance.

\section{Conclusion}
This work is rooted in twin insights: that the reward specification process can be viewed as an iterated game, and that preserving the ability to optimize arbitrary objectives often preserves the ability to optimize the unknown correct objective. To achieve low regret over the course of the game, we can design conservative agents which optimize the primary objective while preserving their ability to optimize auxiliary objectives. We demonstrated how AUP agents act both conservatively and effectively while exhibiting a range of desirable qualitative properties. 

Given our current reward specification abilities, misspecification may be inevitable, but it need not be disastrous.

\section*{Acknowledgments} This work was supported by the Center for Human-Compatible AI and the Berkeley Existential Risk Initiative. We thank Thomas Dietterich, Alan Fern, Adam Gleave, Victoria Krakovna, Matthew Rahtz, and Cody Wild for their feedback, and are grateful for the preparatory assistance of Phillip Bindeman, Alison Bowden, and Neale Ratzlaff.

\appendix 
\section{Theoretical Results}
\label{app:convergence}
Consider an MDP $\langle \mathcal{S},\mathcal{A},T,R,\gamma\rangle$ whose state space $\mathcal{S}$ and action space $\mathcal{A}$ are both finite,
with $\varnothing\in\mathcal{A}$. Let $\gamma\in [0,1)$, $\lambda \geq 0$, and consider finite $\mathcal{R}\subset\mathbb{R}^{\mathcal{S}\times \mathcal{A}}$. 

We make the standard assumptions of an exploration policy greedy in the limit of infinite exploration and a learning rate schedule with infinite sum but finite sum of squares. Suppose $\Call{Scale}{}:\mathcal{S}\to\mathbb{R}_{>0}$ converges in the limit of Q-learning. $\Call{Penalty}{s,a}$ (abbr. $\Call{Pen}{}$), $\Call{Scale}{s}$ (abbr. $\Call{Sc}{}$), and $R_\text{AUP}(s,a)$ are understood to be calculated with respect to the $Q_{R_i}$ being learned online; $\Call{Pen*}{}$, $\Call{Sc*}{}$, $R_\text{AUP}^*$, and $Q^*_{R_i}$ are taken to be their limit counterparts.

\begin{restatable}[]{lem}{pen}
$\forall s,a:\Call{Penalty}{}$ converges with probability 1.
\end{restatable}
\begin{proof}[Proof outline]Let $\epsilon >0$, and suppose for all $R_i\in\mathcal{R}$, $\max_{s,\, a}|Q^*_{R_i}(s,a)-Q_{R_i}(s,a)|<\frac{\epsilon}{2|\mathcal{R}|}$ (because Q-learning converges; see \cite{watkins1992q}).

\begin{align}&\max_{s,\,a}\left|\Call{Penalty*}{s,a}-\Call{Penalty}{s,a}\right|\\
\begin{split}
\leq&\max_{s,\,a}\sum_{i=1}^{|\mathcal{R}|}\left|Q^*_{R_i}(s,a)-Q_{R_i}(s,a)\right|+\\
&\hphantom{\max_{s,\,a}\sum_{i=1}^{|\mathcal{R}|}}\,\,\left|Q^*_{R_i}(s,\varnothing)-Q_{R_i}(s,\varnothing)\right|
\end{split}\\
<&\;\epsilon. 
\end{align}
\end{proof}

The intuition for Lemma \ref{lem:r-aup} is that since $\Call{Penalty}{}$ and $\Call{Scale}{}$ both converge, so must $R_\text{AUP}$. For readability, we suppress the arguments to $\Call{Penalty}{}$ and $\Call{Scale}{}$.

\raup
\begin{proof}[Proof outline] If $\lambda =0$, the claim follows trivially.

Otherwise, let $\epsilon > 0$,  $B \vcentcolon = \max_{s,\,a}\,  \Call{Sc*}{} + \Call{Pen*}{}$, and $C \vcentcolon = \min_{s,\,a} \, \Call{Sc*}{}$. Choose any $\epsilon_R \in \left(0, \min \left[C, \dfrac{\epsilon \, C^2}{\lambda B + \epsilon \, C}\right] \right)$ and assume $\Call{Pen}{}$ and $\Call{Sc}{}$ are both $\epsilon_R$-close.

\begin{align}
&\max_{s,\,a} |R_\text{AUP}^*(s, a) - R_\text{AUP}(s, a)| \\
=&\max_{s,\,a} \lambda \left| \frac{\Call{Pen}{}}{\Call{Sc}{}} - \frac{\Call{Pen*}{}}{\Call{Sc*}{}} \right|\\
=&\max_{s,\,a} \lambda \, \frac{\left|\Call{Pen}{}\cdot\Call{Sc*}{} - \Call{Sc}{}\cdot\Call{Pen*}{}\right|}{ \Call{Sc*}{}\cdot\Call{Sc}{} }  \\
<&\max_{s,\,a} \lambda \, \frac{\left|(\Call{Pen*}{} + \epsilon_R)\Call{Sc*}{} - (\Call{Sc*}{}-\epsilon_R)\Call{Pen*}{}\right|}{C\,(\Call{Sc*}{} - \epsilon_R)}   \\
\leq&\;\frac{\lambda B}{C} \cdot \frac{\epsilon_R}{C - \epsilon_R} \\
<&\;\frac{\lambda B}{C} \cdot \frac{\epsilon \, C^2}{(\lambda B+\epsilon \, C)(C-\frac{\epsilon \, C^2}{\lambda B+\epsilon \, C})} \\
<&\;\frac{\lambda B}{C} \cdot \frac{\epsilon \, C^2}{\lambda B(C-\frac{\epsilon \, C^2}{\lambda B+\epsilon \, C})} \\
<&\;\frac{\epsilon}{1-\frac{\epsilon \, C}{\lambda B+\epsilon \, C}} \\
=&\;\epsilon \left( 1 + \frac{\epsilon \, C}{\lambda B}\right).\label{eq:pf_AUP} 
\end{align}

\noindent But $B, C, \lambda $ are constants, and $\epsilon$ was arbitrary; clearly $\epsilon' > 0$ can be substituted such that $(\ref{eq:pf_AUP}) < \epsilon$.\end{proof} 
\qconv*
\begin{proof}[Proof outline] Let $\epsilon > 0$, and suppose $R_\text{AUP}$ is $\frac{\epsilon (1 - \gamma)}{2} $-close. 
Then Q-learning on $R_\text{AUP}$ eventually converges to a limit $\tilde{Q}_{R_\text{AUP}}$ such that $\max_{s,\,a} |Q^*_{R_\text{AUP}}(s, a) - \tilde{Q}_{R_\text{AUP}}(s, a)| < \frac{\epsilon}{2}$. 
By the convergence of Q-learning, we also eventually have $\max_{s,\,a} |\tilde{Q}_{R_\text{AUP}}(s, a) - Q_{R_\text{AUP}}(s, a)| < \frac{\epsilon}{2}$. Then 
 
\begin{align}
& \max_{s,\,a} \left|Q^*_{R_\text{AUP}}(s, a) - Q_{R_\text{AUP}}(s, a)\right| < \epsilon.
\end{align}
\end{proof}
\begin{prop}[Invariance properties]

    Let $c \in \mathbb{R}_{>0}, b \in \mathbb{R}$.
    \begin{enumerate}
    \item[a)] Let $\mathcal{R}'$ denote the set of functions induced by the positive affine transformation $cX + b$ on $\mathcal{R}$,    
    and take $\Call{Pen*}{}_{\mathcal{R}'}$ to be calculated with respect to attainable set $\mathcal{R}'$. Then $\Call{Pen*}{}_{\mathcal{R}'} = c \, \Call{Pen*}{}_{\mathcal{R}}$. 
    In particular, when $\Call{Sc*}{}$ is a $\Call{Penalty}{}$ calculation, $R^*_{\text{AUP}}$ is invariant to positive affine transformations of $\mathcal{R}$.   \\
    \item[b)] Let $R':= c R + b$, and take $R'^*_{\text{AUP}}$ to incorporate $R'$ instead of $R$.
    Then by multiplying $\lambda $ by $c$, the induced optimal policy remains invariant. 
    \end{enumerate}
\end{prop}
    
\begin{proof}[Proof outline]
    For a), since the optimal policy is invariant to positive affine transformation of the reward function, for each $R'_i \in \mathcal{R}'$ we have $Q^*_{R'_i} = c \, Q^*_{R_i} + \frac{b}{1-\gamma}$. Substituting into Equation \ref{penalty} ($\Call{Penalty}{}$), the claim follows.\\

    \noindent For b), we again use the above invariance of optimal policies:
    \begin{eqnarray}
        R'^*_{\text{AUP}} &:=& c R + b - c\lambda \,\frac{\Call{Pen*}{}}{\Call{Sc*}{}} \\
        &=& c R^*_{\text{AUP}} + b.
    \end{eqnarray}

\end{proof}

\bibliographystyle{named}
\bibliography{AI_Safety.bib}

\end{document}